%% file: main.tex
\documentclass[twocolumn]{article}
\usepackage{times} 
\usepackage{helvet}  
\usepackage{courier}  
\usepackage[hyphens]{url}  
\usepackage{graphicx} 
\usepackage[a4paper, total={7in, 10in}]{geometry}
\usepackage{caption} 
\usepackage[square,numbers]{natbib}
\usepackage{algpseudocodex}           
\usepackage{algorithm} 
\usepackage{booktabs}

\usepackage{tikz}
\usepackage{environ}
\usepackage{float}
\usepackage{multirow}
\usepackage{xcolor}
\usepackage{dirtytalk}
\usepackage{amsmath}
\usepackage{amssymb}
\usepackage{amsthm}
\usepackage{hyperref}
\usepackage{url}
\usepackage{microtype}
\usepackage{relsize}

\definecolor{my_red}{HTML}{e63946}
\definecolor{my_beige}{HTML}{f1faee}
\definecolor{my_lightblue}{HTML}{a8dadc}
\definecolor{my_blue}{HTML}{457b9d}
\definecolor{my_darkblue}{HTML}{1d3557}
\definecolor{my_orange}{HTML}{f08700}

\newtheorem{lemma}{Lemma}
\newtheorem{theorem}{Theorem}

\newtheorem{definition}{Definition}

\newcommand{\norm}[1]{\left\lVert#1\right\rVert}
\newcommand{\E}[0]{\mathbb{E}}
\title{Bias-Aware Minimisation: Understanding and\\ Mitigating Estimator Bias in Private SGD}
\author {
    Moritz Knolle\textsuperscript{\rm1,2},
    Robert Dorfman\textsuperscript{\rm3},
    Alexander Ziller\textsuperscript{\rm1},
    Daniel Rueckert \textsuperscript{\rm1,2,4} and
    Georgios Kaissis\textsuperscript{\rm1,2,4} \\
    \textsuperscript{\rm 1} \textit{Institute for AI in Medicine, Technical University of Munich}\\
    \textsuperscript{\rm 2} \textit{Konrad Zuse School for Excelllence in Reliable AI}\\
    \textsuperscript{\rm 3} \textit{V7 Labs}\\
    \textsuperscript{\rm 4} \textit{Imperial College London}\\
    }

\begin{document}

\maketitle

\begin{abstract}
Differentially private SGD (DP-SGD) holds the promise of enabling the safe and responsible application of machine learning to sensitive datasets. However, DP-SGD only provides a biased, noisy estimate of a mini-batch gradient. This renders optimisation steps less effective and limits model utility as a result. With this work, we show a connection between per-sample gradient norms and the estimation bias of the private gradient oracle used in DP-SGD. Here, we propose Bias-Aware Minimisation (BAM) that allows for the provable reduction of private gradient estimator bias. We show how to efficiently compute quantities needed for BAM to scale to large neural networks and highlight similarities to closely related methods such as Sharpness-Aware Minimisation. Finally, we provide empirical evidence that BAM not only reduces bias but also substantially improves privacy-utility trade-offs on the CIFAR-10, CIFAR-100, and ImageNet-32 datasets.
\end{abstract}

\section{Introduction}
The application of machine learning models to datasets containing sensitive information necessitates meaningful privacy guarantees, as it has been shown that the public release of unsecured models trained on sensitive data can pose serious threats to data-contributing individuals \cite{geiping2020inverting, carlini2021extracting, carlini2019secret}.
To this end, \textit{differential privacy} (DP) \cite{dwork2014algorithmic}, the widely accepted \textit{gold standard} approach to privacy-preserving data analysis, has been extended to many machine learning methods. Most significantly, for modern machine learning, differentially private stochastic gradient descent (DP-SGD) \cite{shokri2015privacy, abadi2016deep} has enabled the application of powerful deep neural networks to sensitive datasets with meaningful privacy guarantees.

In practice, however, DP-SGD comes with a substantial utility penalty. As we demonstrate below, this is because DP-SGD only provides a noisy, biased gradient estimate, causing optimization to settle in regions of the loss landscape that yield poorly performing models.
We show that (ignoring sampling noise), the variance of a private gradient estimate is fixed while its bias is not.
In this work, we attempt to improve private gradient estimators by minimising their bias. 
The bias of an estimate $\hat{p}$ of the true value $p$ underlying a process generating observations $x$ is defined as:
\begin{equation}
    \mathrm{Bias}(\hat{p}, p) =
    \E_{x}[\hat{p}]-p.
    \label{eq:bias}
\end{equation}

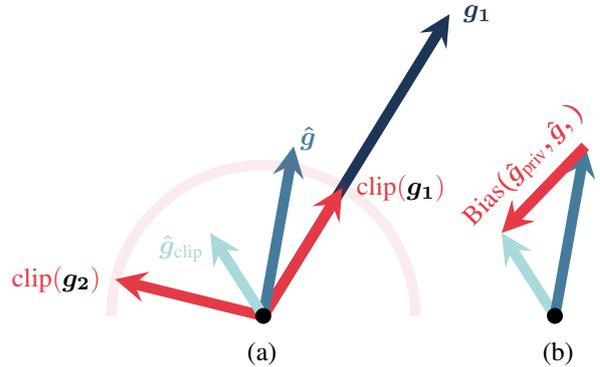
\begin{figure}[t]
\centering
    \begin{minipage}[t]{0.3\linewidth}
        \centering
        \input{figs/clip_a}
    \end{minipage}
    \hspace{30mm}
    \begin{minipage}[t]{0.3\linewidth}
        \centering
        \input{figs/clip_b}
    \end{minipage}
    \caption{\textbf{(a)} Clipping per-sample gradients $g_1$ and $g_2$ yields a biased gradient estimate \textcolor{my_lightblue}{$\boldsymbol{\hat{g}_{\mathrm{clip}}}$} of the true mini-batch gradient \textcolor{my_darkblue}{$\boldsymbol{\hat{g}}$} in expectation. Below we show this bias, represented by the vector \textcolor{my_red}{$\boldsymbol{\text{Bias}(\hat{g}_\text{priv}, \hat{g},)}$} in \textbf{(b)}, to depend on $\lVert g_1\rVert_2$ and $\lVert g_2\rVert_2$.
    }
    \label{fig:clip_bias}
\end{figure}
Thus, when an estimator has zero bias, it is called \textit{unbiased} and its expected value is equal to the quantity being estimated. 
When estimating gradients, unbiasedness is a useful property as it implies that gradient estimates --in expectation-- point towards the correct descent direction.

Bias in DP-SGD results from the per-sample gradient clipping operation, used to enforce a sensitivity bound on the gradients (illustrated above in Fig. \ref{fig:clip_bias}).
However, despite its importance, the nature of this bias and how it affects parameter updates has remained poorly understood. With this work, we show that private gradient estimator bias is intimately related to the per-sample gradient norms and can further be decomposed into a \textit{magnitude} and \textit{directional} component (See Appendix \ref{sec:decomp}). 
We develop \textbf{B}ias-\textbf{A}ware \textbf{M}inimisation (\textbf{BAM}) and provide empirical evidence that our method effectively (I) reduces bias magnitude and (II) increases classification performance on a range of challenging vision benchmarks.

\section{Related Work}
The bias-variance tradeoff of the private gradient oracle was first discussed in \cite{mcmahan2018learning}.
More generally, bias associated with gradient clipping in SGD has been studied by \citet{zhang2019gradient, qian2021understanding} and also in the private setting, where \citet{song2020characterizing} demonstrated the importance of the correct choice of clipping threshold in convex models, while \citet{chen2020understanding} showed that no bias is incurred when a symmetricity condition on the gradient distribution holds.

Concurrent work \cite{park2023DPSharpness, shi2023Flatter} has provided empirical evidence that Sharpness-Aware-Minisation (SAM) \cite{foret2020sharpness} can improve privacy-utility tradeoffs. 
\citet{park2023DPSharpness} suggest that the magnitude of the bias vector (referred to as the \textit{effect of clipping} in their work) is bounded by the sharpness. 
With this work, we provide evidence that sharpness, while closely related, might \textit{not be the root cause determining private gradient bias} but rather the per-sample gradient norms in a mini-batch.
 
\section{Preliminaries}

\textbf{Differential privacy} \cite{dpbook} (DP), is a stability notion on randomised mechanisms over sensitive databases.
Let $D$ and $D'$ be two databases (datasets) that differ exactly in one individual's data.
We denote this relationship (i.e. database adjacency) through the symbol $\simeq$ and use the standard remove/add one relation throughout.
\begin{definition}[Differential privacy]
A randomised mechanism $M$ executed on the result of a query function $q$ (i.e. $M(q(\cdot))$) preserves $(\varepsilon$, $\delta)$-DP if, for all pairs of adjacent databases $D$ and $D'$, and all subsets $\mathcal{S}$ of $M(q(\cdot))$'s range:
    \begin{equation}
        p\left[M\left(q(D)\right) \in \mathcal{S}\right] \leq \mathlarger{e^{\varepsilon}} \, p\left[M(q(D')) \in \mathcal{S}\right] + \delta,
    \end{equation}
where the relationship between $D$ and $D'$ is symmetric.
The guarantee is given over the randomness of $M$. 
\end{definition}
To achieve $(\varepsilon, \delta)$-DP using Gaussian noise perturbation (i.e. the Gaussian mechanism $\mathcal{M}$), one needs to calibrate the magnitude of the noise to the query function's (global) sensitivity:
\begin{equation*}
    \Delta_2(q) =  \sup_{D  \simeq D'}\norm{q(D)-q(D')}_2,
\end{equation*}
or $\Delta$ for short.
Since obtaining a (non-vacuous) sensitivity bound is generally not possible for deep neural network gradients, it is common practice to enforce a bound manually by projecting per-sample gradients to the $L_2$-ball \cite{abadi2016deep}.
Together, this operation, more commonly known as \textit{clipping}, followed by appropriately scaled Gaussian perturbation, then yields the DP-SGD algorithm, which can be thought of as simply querying a private gradient oracle at each optimisation step as seen in Algorithm \ref{alg:oracle}: 

\begin{algorithm}[h]
\caption{Private gradient oracle $\psi$}
\textbf{Input:} (Poisson-sampled) mini-batch $B = \{(x_1, y_1), ..., (x_{l}, y_{l}) \}$, Clipping Bound $C$, noise multiplier $\sigma$, loss function $\mathcal{L}$, parameters $\theta \in \mathbb{R}^d$\\
\textbf{Output}; private gradient estimate

\begin{algorithmic}[1]
    \For{$(x_i, y_i)$ in $B$}
        \State $g_i \gets \nabla_\theta \mathcal{L}\left(f_\theta(x_i), y_i\right)$
        \State $\bar{g}_i \gets g_i / \max(1, \frac{\norm{g_i}_2}{C})$  \Comment{clip}
    \EndFor
    \State $\hat{g}_{\mathrm{priv}} \gets \frac{1}{l} \left( \sum_{i=1}^{l} \bar{g}_i + \mathcal{N}(0, \sigma^2 C^2 I_d) \right)$ \Comment{perturb}
    \State \textbf{return} $\hat{g}_{\mathrm{priv}}$
\end{algorithmic}
\label{alg:oracle}
\end{algorithm}

Where $l$ is the expected batch size due to the sub-sampling amplification requirement that $B$ is constructed through a Poisson sample of the Dataset $D$.

\subsection{Setting}
We focus on supervised learning where, given a dataset $D = \{(x_i, y_i), ..., (x_n, y_n)\}$ drawn i.i.d from a product distribution $\mathcal{X} \times \mathcal{Y}$, we wish to find a mapping $f: \mathcal{X} \rightarrow \mathcal{Y}$, realised through a neural network.
This neural network has parameters $\theta \in \Theta \in \mathbb{R}^d$ and is trained by minimising the empirical loss function $\mathcal{L}(\theta) := \frac{1}{n}\sum_{i = 1}^n\mathcal{L}(\theta, x_i, y_i)$ using DP-SGD, whereby at each optimisation step the private gradient oracle is queried to obtain a privatised gradient estimate.

\begin{align}
    \theta^{(t+1)} = \theta^{(t)} - \gamma^{(t)} \hat{g}^{(t)}_{\mathrm{priv}}\\
     \hat{g}^{(t)}_{\mathrm{priv}} =  \psi(B, C, \sigma, \mathcal{L}, \theta^t)
\end{align}

We drop the superscript $t$ for notational simplicity, assuming we are on step $t$ in the following analyses. We further denote the clipped minibatch gradient as $\hat{g}_{\mathrm{clip}}= 1/l \sum_i^l \bar{g}_i$.

\section{Clipping dominates private gradient bias}

\label{section:clipping_bad}
Since we are interested in obtaining an unbiased version of DP-SGD, we first study the bias introduced by constructing a private estimate for the mini-batch gradient $\hat{g}$. Our analysis focuses on the quantity:
\begin{equation*}
    \text{Bias}\left(\hat{g}_{\mathrm{priv}}, \, \hat{g}\right) = \E[\hat{g}_{\mathrm{priv}}] - \hat{g},
\end{equation*}
where we view $\hat{g}$ as fixed, that is, the gradient constructed for an already observed batch of data to be used in a step of conventional SGD.
This enables us to isolate the bias introduced through private estimation of $\hat{g}$. 
Note that when using $\E$ without a subscript, we take the expectation over all randomness present.
In the case of $\hat{g}_{\mathrm{priv}}$, this randomness is --by assumption-- only due to the Gaussian mechanism.
We first observe that this gradient perturbation does not introduce additional bias:

\begin{lemma}
\label{noise lemma}
The bias of the private gradient estimate, Bias($\hat{g}_{\mathrm{priv}}, \hat{g})$, is unaffected by the noise addition in the Gaussian mechanism. That is,
$$\mathrm{Bias}(\hat{g}_{\mathrm{priv}}, \hat{g}) = \mathrm{Bias}(\hat{g}_{\mathrm{clip}}, \hat{g})$$
\end{lemma}

\begin{proof}
\begin{align}
    \mathrm{Bias}(\hat{g}_{\mathrm{priv}},& \hat{g})= \E\Big[\frac{1}{l}\sum_{i=1}^l\Big(\mathrm{clip}(g_i) + \mathcal{N}(0, \sigma^2C^2I_d)\Big)\Big] - \hat{g}  \\
    =& \frac{1}{l}\sum_{i=1}^l \Big(\mathrm{clip}(g_i) + \E\big[\mathcal{N}(0, \sigma^2 C^2I_d)\big]\Big) - \hat{g} \\
    =& \hat{g}_{\mathrm{clip}} - \hat{g} \\
    =& \mathrm{Bias}(\hat{g}_{\mathrm{clip}}, \hat{g})
    \label{bias g_priv}
\end{align}

\noindent This holds since $\hat{g}$ and $\hat{g}_{\mathrm{clip}}$ are viewed as constructed from an observed mini-batch of data and the zero centred Gaussian random variables are independent of one another.
\end{proof}
Thus, private gradient bias is caused by the clipping operation alone.
Next, we develop an objective function that provably minimises the aforementioned bias.


\section{A bias-aware objective}
\label{sec:BAO}
We propose $\mathcal{L}_{\mathrm{BAO}}$, an objective that, when minimised, provably reduces the bias of private gradient estimates $\hat{g}_{\mathrm{priv}}$ by encouraging small per-sample gradient norms:
\begin{equation}
\label{bias aware objective}
\mathcal{L}_{\mathrm{BAO}}(\theta, x, y) = \underbrace{\mathcal{L}(\theta, x, y)}_{\text{original loss}} + \underbrace{\lambda \left(\frac{1}{l}\sum_{i=1}^l\norm{g_i}_2\right)}_{\text{regularising term}}
\end{equation}
which can be sub-sampled as follows:
\begin{equation}
    \mathcal{L}_{\mathrm{BAO}}(\theta, x_i, y_i) = \mathcal{L}(\theta, x_i, y_i) + \lambda \norm{g_i}_2  
\end{equation}
To motivate this optimisation objective, we will now demonstrate the primary dependence of Bias($\hat{g}_{\mathrm{priv}}, \hat{g})$ on the per-sample gradient norms for a fixed clipping threshold $C$.
\begin{lemma}
\label{bias lemma}
A smaller per-sample gradient norm $\lVert g_i\rVert_2$ of the $i$-th sample in a mini-batch decreases $\text{Bias}(\hat{g}_{\mathrm{priv}}, \hat{g})$.
\end{lemma}

\begin{proof}
First note that
\[
\mathrm{clip}(g_i) = \frac{g_i}{\mathrm{max}(1,\frac{\lVert g_i\rVert_2)}{C})}  = 
\begin{cases} 
      g_i & \mathrm{if}\, \lVert g_i\rVert_2 \leq C \\
      \frac{Cg_i}{\lVert g_i\rVert_2} & \mathrm{if}\, \lVert g_i\rVert_2 > C.
\end{cases}
\]

\noindent Thus, if $\lVert g_i\rVert_2 \leq C$ for every $i$ in the mini-batch, then by Lemma \ref{noise lemma}:
$$\mathrm{Bias}(\hat{g}_{\mathrm{priv}}, \hat{g}) = \frac{1}{l}\sum_{i=1}^l g_i - \hat{g} = 0.$$
\noindent That is, if the $L_2$-norm of the gradients of the mini-batch are all below the clipping threshold, the bias of the private gradient estimate reduces to zero. On the other hand, if $\lVert g_i\rVert_2 > C$:

\begin{flalign}
\label{greater than thresh gpriv}
\mathrm{Bias}(\hat{g}_{\mathrm{priv}}, \hat{g}) &= \frac{1}{l}\sum_{i=1}^l\frac{C}{\lVert g_i\rVert_2}g_i - \hat{g}.
\end{flalign}

\noindent Therefore, the bias is dependent on the ratios $\frac{C}{\lVert g_i\rVert_2} \forall i$: as a per-sample gradient norm $\lVert g_i\rVert_2$ gets increasingly larger than $C$, the bias grows (note that here \say{grows} means gets further from the origin). Conversely, if we minimise $\lVert g_i\rVert_2$ for any $i$, the bias will shrink until all $\norm{g_i}_2$ in the mini-batch are less than $C$, at which point the estimate becomes unbiased. 
\end{proof}

Notice that \eqref{bias aware objective} is similar to the gradient norm penalty, used by \cite{zhao2022penalizing} to encourage flatness of the loss landscape:
\begin{equation}
    \label{zhao_objective}
    \mathcal{L}_{Z}(\theta) = \mathcal{L} (\theta) + \norm{\nabla_\theta \mathcal{L}(\theta)}_2,
\end{equation}
It can be shown that \eqref{bias aware objective} upper-bounds $\mathcal{L}_Z$ (see Appendix \ref{sec:bam_flatness}).
This means that a reduction in $\mathcal{L}_{\mathrm{BAO}}$ implies a reduction in $\mathcal{L}_Z$, but the converse is not generally true.
Note also that $\mathcal{L}_Z$ does not allow for the required per-sample DP analysis, precluding the simple application of DP-SGD \cite{park2023DPSharpness}.


\begin{table*}[htbp]
    \centering
    \begin{minipage}[c]{0.9\textwidth}
        \centering
        \resizebox{0.7\linewidth}{!}{
        \begin{tabular}{llllll}
        \toprule
        \textbf{Dataset}            & $\boldsymbol{\varepsilon}$ & $\boldsymbol{\delta}$& \textbf{DP-SGD}   & \textbf{DP-SAT}   & \textbf{BAM} (ours) \\ \midrule
        \multirow{3}{*}{CIFAR-10}   &   $1.0$       &\multirow{3}{*}{$10^{-5}$} &  $60.9 \pm 0.49$ &  $60.9 \pm 0.62$   &  $\mathbf{61.4 \pm 0.48}$ \\
                                    &   $2.0$       &                           &  $67.1 \pm 0.10$ &  $67.2 \pm 0.30$   &   $\mathbf{68.2 \pm 0.27}$ \\
                                    &   $10.0$      &                           &  $78.6 \pm 0.08$ &  $78.1 \pm 0.69$   &   $\mathbf{79.7 \pm 0.13}$ \\
                                    \midrule
        \multirow{3}{*}{CIFAR-100}  &   $1.0$       &\multirow{3}{*}{$10^{-5}$} &  $18.1 \pm 0.10$ & $18.2 \pm 0.13$    &   $\mathbf{18.5 \pm 0.04}$ \\
                                    &   $2.0$       &                           &  $24.9 \pm 0.46$ & $24.9 \pm 0.35$    &   $\mathbf{25.4 \pm 0.40}$ \\
                                    &   $10.0$      &                           & $40.3 \pm 0.21$  & $40.1 \pm 0.19$    &   $\mathbf{40.8 \pm 0.06}$ \\
                                    \midrule
        ImageNet32                  &   $10.0$       &$8\times10^{-7}$ & $14.97$  &  $14.70$ &   $\mathbf{20.67}$ \\
        \bottomrule
        \end{tabular}
        }
        \vspace{4mm}
        \caption{Test accuracy (mean$\pm$SD \%) for \textit{CIFAR-10}, \textit{CIFAR-100} and \textit{ImageNet32} computed over three random seeds at different $(\varepsilon, \delta)$. Due to computational resource constraints, we report only a single training run for Imagenet32.}
        \label{table:results}
        \vfill
    \end{minipage}
    
\end{table*}

\subsection{Efficient Computation}
\label{eff_compute}
Na\"ively implementing $\mathcal{L}_{\mathrm{BAO}}$ is problematic, as the gradient computation now involves computing a Hessian-Vector Product (HVP) for every sample:
\begin{align}
    \label{bao_gradient}
    \nabla_\theta \mathcal{L}_{\mathrm{BAO}}(\theta, x, y) &= \nabla_\theta \mathcal{L}(\theta, x, y) + \notag \\ 
    \lambda \frac{1}{n} \sum_{i=1}^l & \underbrace{\nabla^2_{\theta} \mathcal{L}(\theta, x_i, y_i) \frac{\nabla_\theta\mathcal{L}(\theta, x_i, y_i)}{\norm{\nabla_\theta \mathcal{L}(\theta, x_i, y_i)}_2}}_{\text{HVP}}.
\end{align}
This objective is thus expensive to compute for deep networks using reverse-mode automatic differentiation (AD) \cite{karakida2022understanding}. Fortunately, however, prior work has shown that HVPs can be computed efficiently, in this case, either exactly using a combination of forward and reverse mode AD \cite{autodiffcookbook, pearlmutter1994fast}, or approximately using the local gradient ascent step of SAM \cite{zhao2022penalizing}.

\section{Method}
To reduce the bias of the private gradient oracle in  DP-SGD, we optimise our bias-aware objective $\mathcal{L}_{\mathrm{BAO}}$ and approximate the necessary per-sample gradients.
Concretely, we perform the local gradient ascent step of SAM, at the sample level before computing the gradient:
\begin{equation}
    \nabla_\theta \mathcal{L}_{\mathrm{BAO}}(\theta, x_i, y_i) \approx \nabla_{\theta} \mathcal{L}(\theta, x_i, y_i)\bigg |_{\theta= \theta + \lambda \frac{\nabla_\theta \mathcal{L}(\theta, x_i, y_i)}{\norm{\nabla_\theta \mathcal{L}(\theta, x_i, y_i)}_2}}.
\end{equation}

The entire training procedure is summarised in Algorithm \ref{alg:db_dpsgd_algo}.

\begin{algorithm}[h]
    \caption{\textbf{Bias-Aware Minimisation} (BAM)}
    \label{alg:db_dpsgd_algo}
    \begin{algorithmic}[1]
       \For{$t\in 1, 2, ..., T$}
            \State $B \gets$ Poisson sample of $D$ with probability $q$
            \For{$(x_i, y_i) \in B$}
            \State $\theta'^{(t)} \gets \theta^{(t)} + \lambda^{(t)} \frac{\nabla_{\theta}  \mathcal{L}(\theta^{(t)}, x_i, y_i)}{\norm{\nabla_\theta \mathcal{L}(\theta^{(t)}, x_i, y_i)}_2}$ \Comment{\small SAM step}
            \State $g^{(t)}_i \leftarrow \nabla_{\theta'}\mathcal{L}_{\mathrm{BAO}}(\theta'^{(t)}, x_i, y_i)$
            \State $\bar{g}_i^{(t)} \leftarrow g_i^{(t)} / \max \Big(1, \, \frac{\norm{g_i^{(t)}}_2}{C}\Big)$ \Comment{\small clip}
            \EndFor
            \State $\hat{g}_{\mathrm{priv}}^{(t)} = \frac{1}{l} \sum_i^l\left[{\bar{g}^{(t)}_i} +\mathcal{N}(0,\, \sigma^2C^2 I_d) \right]$ \Comment{\small perturb}
            \State $\theta^{(t+1)} = \theta^{(t)} - \gamma^{(t)} \, \hat{g}_{\mathrm{priv}}^{(t)}$
        \EndFor
    \end{algorithmic}
\end{algorithm}

Note that, since $\nabla_{\theta} \mathcal{L}_{\mathrm{BAO}}(\theta, x_i, y_i)$ only depends on sample-level statistics, the sampling process is identical to DP-SGD and all quantities are privatised as in DP-SGD, every iteration $t$ of Algorithm \ref{alg:db_dpsgd_algo} satisfies $(\varepsilon, \delta)$-DP with identical privacy parameters as DP-SGD.

\section{Results}
To evaluate our proposed approach and compare its performance to DP-SGD and DP-SAT \cite{park2023DPSharpness}, we perform a range of experiments on challenging computer vision datasets. Results for  \textit{CIFAR-10}/\textit{100} \cite{krizhevsky2009learning} and \textit{ImageNet32}\cite{chrabaszcz2017downsampled} are reported above in Table \ref{table:results}. We employ state-of-the-art (SOTA) training practices for DP-SGD \cite{de2022unlocking}, namely: weight standardisation, group normalisation, large batch sizes and augmentation multiplicity. Full experimental details and hyperparameter values can be found in Table \ref{tab:hyper}, Appendix \ref{sec:further_details}.

\begin{figure}[h]
    \centering
    \includegraphics[width=0.8\linewidth]{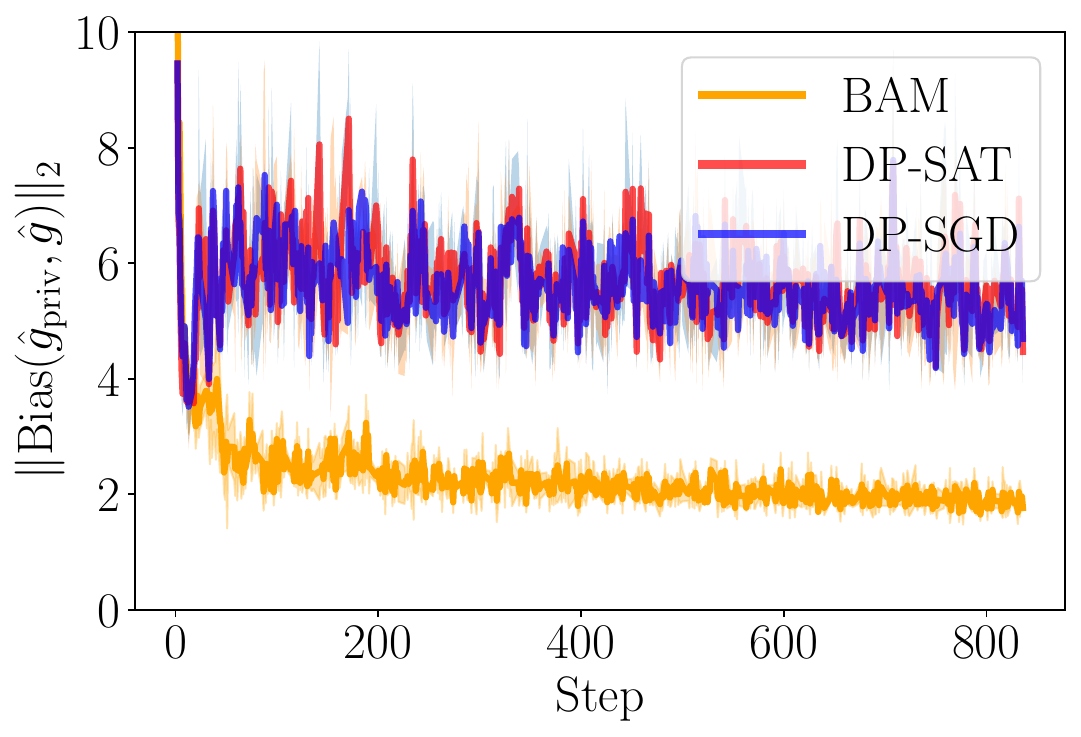}
    \vspace{-2mm}
    \captionof{figure}{BAM effectively minimises bias: The magnitude of the bias vector measured on CIFAR-10 at a batch size of $512$ is substantially lower for BAM, while DP-SAT incurs the same bias as DP-SGD.
    }
    \label{fig:bias}
\end{figure}

We find that BAM effectively minimises private gradient bias in practice (see Fig. \ref{fig:bias} above). 
On the other hand, DP-SAT largely has no effect on estimator bias for values of the regularisation parameter $\lambda$ that yield high-performing models.
Further empirical run-time comparisons (data shown in Appendix \ref{sec:further_details}) reveal that while both exact and approximate gradient computations for BAM do incur higher computational costs than DP-SAT (and DP-SGD), this burden is manageable for most practically sized networks with less than 200 layers.
Finally, our empirical performance comparison reveals that when using very large batch sizes and the other SOTA practices of \citet{de2022unlocking}, performance gains on more challenging datasets realised through DP-SAT are smaller than reported in the original publication \cite{park2023DPSharpness}. 
In contrast, BAM consistently improves performance across different privacy budgets (Table \ref{table:results}).

\section{Discussion}

After deriving the bias of the private gradient oracle from first principles, we developed a bias-aware regularisation objective $\mathcal{L}_{\mathrm{BAO}}$ that was empirically confirmed effective at minimising the bias vector associated with private gradient estimation.
We further demonstrated that using the SAM approximation as suggested by \cite{zhao2022penalizing}, our objective and its gradient are computable with manageable computational overhead. 

Our method performs the local gradient ascent step (SAM step) at the per-sample level and is thus closely related to the DP-SAT method of \cite{park2023DPSharpness}, which performs the ascent step with the previous iteration's privatised mini-batch gradient.
Our experiments show that our method outperforms DP-SAT at minimising bias, which is mirrored in the superior accuracy of models trained with BAM.
Notably, our method yields a more than $5\%$ accuracy increase on the most challenging dataset tested, ImageNet32.
Based on our results, we hypothesise that, with modern private training practices \cite{de2022unlocking} (very large batch size, large learning rate), the noisy gradient of the previous iteration ($\hat{g}^{(t-1)}_{\mathrm{priv}}$), is a poor approximation for finding the maximum loss in the local neighbourhood around the current iteration's parameter value $\theta^{(t)}$.
This is corroborated by a simple experiment (data shown in Fig. \ref{fig:eff_ascent}, Appendix \ref{sec:further_details}) that shows the ascent step vector used in DP-SAT to point in slightly different directions to the current iteration's (un-privatised) mini-batch gradient conventionally used in SAM. 

Our results, while not without limitations (BAM incurs slight computational overhead), have shown that reducing private gradient bias can lead to effective performance increases on challenging image datasets. More extensive empirical evaluation, especially on large-scale datasets such as (full-size) ImageNet, is ongoing work and required to fully assess its benefits. 
Finally, future work should investigate the effectiveness of alternative methods that encourage smoothness \cite{cha2021swad} and look into connections to and impact on model fairness with respect to sub-groups \cite{tran2021differentially}.

\section{Acknowledgments}
This paper was supported by the DAAD programme Konrad Zuse Schools of Excellence in Artificial Intelligence, sponsored by the Federal Ministry of Education and Research.

\bibliography{main}

\appendix

\begin{table*}[htbp]
    \centering
    \begin{tabular}{llllllll}
        \textbf{Dataset}     & \textbf{Method}    & $\lambda$     & \textbf{Model}     & \textbf{Batch-size}    &  \textbf{LR} &  C & \textbf{Multiplicity}  \\ \toprule
        \multirow{3}{*}{CIFAR-10}   & DP-SGD    & $0.0$  & \multirow{3}{*}{ResNet-9}   &  \multirow{3}{*}{$32\times128$} & \multirow{3}{*}{2e-3}  & \multirow{3}{*}{$1$}  &    \multirow{3}{*}{16}                    \\
           & DP-SAT    & $0.086$  &    &   &   &   &           \\
           & BAM    & $0.02$  &    &   &   &   &           \\ \midrule
        \multirow{3}{*}{CIFAR-100}   & DP-SGD    & $0.0$  & \multirow{3}{*}{ResNet-9}   &  \multirow{3}{*}{$32\times128$} & \multirow{3}{*}{2e-3}  & \multirow{3}{*}{$1$}  &    \multirow{3}{*}{16}                    \\
           & DP-SAT    & $0.086$  &    &   &   &   &           \\
           & BAM    & $0.01$  &    &   &   &   &           \\ \midrule
        \multirow{3}{*}{ImageNet32}   & DP-SGD    & $0.0$  & \multirow{3}{*}{WideResNet-16-4}   &  \multirow{3}{*}{$8\times512$} & \multirow{3}{*}{1e-3}  & \multirow{3}{*}{$1$}  &    \multirow{3}{*}{4}                    \\
           & DP-SAT    & $0.07$  &    &   &   &   &           \\
           & BAM    & $0.005$  &    &   &   &   &           \\
         \bottomrule
    \end{tabular}
    \caption{Hyperparameter values for experiments on CIFAR-10, CIFAR-100 and ImageNet32.}
    \label{tab:hyper}
\end{table*}


\section{Experimental details \& further results}
\label{sec:further_details}
To obtain suitable hyperparameter values for the compared methods (reported above in Table \ref{tab:hyper}), a random search was first employed using DP-SGD, after which a separate search for optimal $\lambda$-values was performed for both DP-SAT and BAM with N=200 random trials each. For Imagenet32, only a very small-scale hyperparameter search was employed due to the high computational cost. Augmentation multiplicity, as described in \cite{de2022unlocking}, was employed with probability $p=0.5$ (for each augmentation) across all compared methods with random pixel shifts (up to $4$ pixels) and random vertical flips. All models were trained for $75$ epochs with the NAdam\cite{dozat2016incorporating} optimizer and, besides the learning rate (LR), otherwise, default hyperparameters. To meet the different $(\varepsilon, \delta)$ privacy budget requirements the noise multiplier $\sigma^2$ was adjusted accordingly.

\subsection{Computational aspects}
\begin{figure}[H]
    \centering
    \begin{minipage}{0.55\linewidth}
        \includegraphics[width=\linewidth]{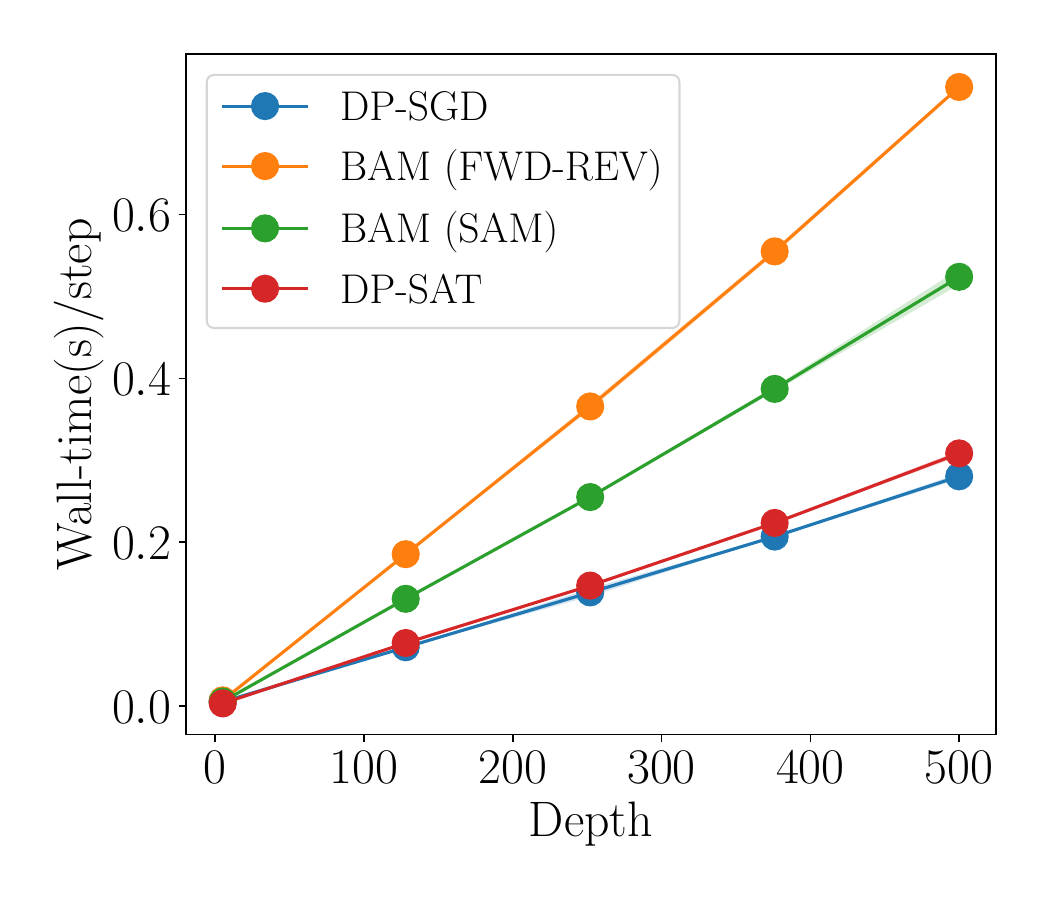}
    \end{minipage}
    \vspace{-3mm}
    \caption{Per-step wall-time (mean $\pm$ SD) for different computational approaches to bias mitigation in DP-SGD for increasing network depth. Results were computed over ten trials and five repetitions.}
    \label{fig:bench_time}
\end{figure}

We evaluate the empirical run-time complexity of the two previously mentioned approaches to compute $\nabla_\theta \mathcal{L}_{\mathrm{BAO}}(\theta, x_i, y_i)$. Concretely, we compare forward-over-reverse mode AD (BAM FWD-REV) and the SAM (BAM SAM) approximation to compute necessary gradients for BAM and compare to DP-SAT \cite{park2023DPSharpness}. DP-SAT uses the previous step's privatised mini-batch gradient to approximate \eqref{zhao_objective} at no additional privacy cost and little computational overhead. Figure \ref{fig:bench_time} showcases the results of an empirical run-time comparison on a toy dataset, implemented in \texttt{jax} \cite{jax2018github} and compiled with XLA to obtain a fair comparison between methods.

\subsection{Gradient ascent step effectiveness}
To investigate the effectiveness the gradient ascent step in DP-SAT and BAM, we investigate the cosine similarity between the ascent direction used by the respective method and the (ground-truth) ascent direction as used in SAM, that is the un-privatised, current mini-batch gradient. Our findings (below in Fig. \ref{fig:eff_ascent}) indicate that the per-sample gradient ascent step of DP-SAT is substantially better aligned with the non-private SAM ascent step than the  ascent step with the previous iteration's privatised gradient of the DP-SAT method. Concretely, we find DP-SAT ascent steps to point in slightly opposite directions, as indicated by a small, but negative cosine similarity.

\begin{figure}[h]
    \centering
    \includegraphics[width=0.7\linewidth]{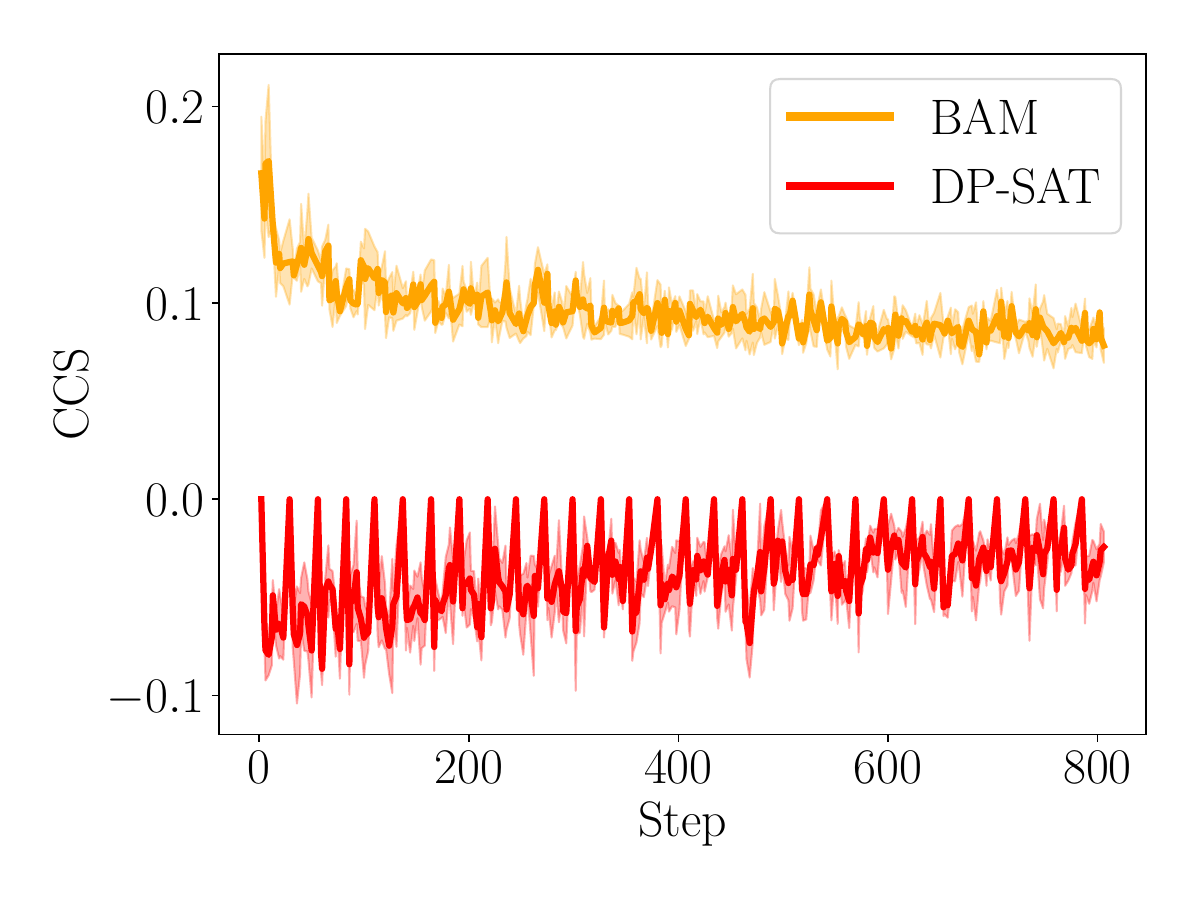}
    \vspace{-3mm}
    \caption{Effectiveness of the ascent step throughout training on CIFAR-10 (three repetitions) with DP-SAT and BAM as measured by the cosine similarity $\mathrm{cos}(\hat{g}_{\mathrm{priv}}^{(t-1)}, \hat{g}^{(t)})$ and $1/l\sum_i^l\mathrm{cos}(g_i^{(t)}, \hat{g}^{(t)})$ respectively.}
    \label{fig:eff_ascent}
\end{figure}

\subsection{C. Batch size and its effects on bias and directional error}
To investigate the empirical success of recent works using  extremely large batch sizes (or even full-batch training), we investigate the impact of batch size on private gradient bias and its directional component. 

\begin{figure}[H]
    \centering
    \includegraphics[width=0.65\linewidth]{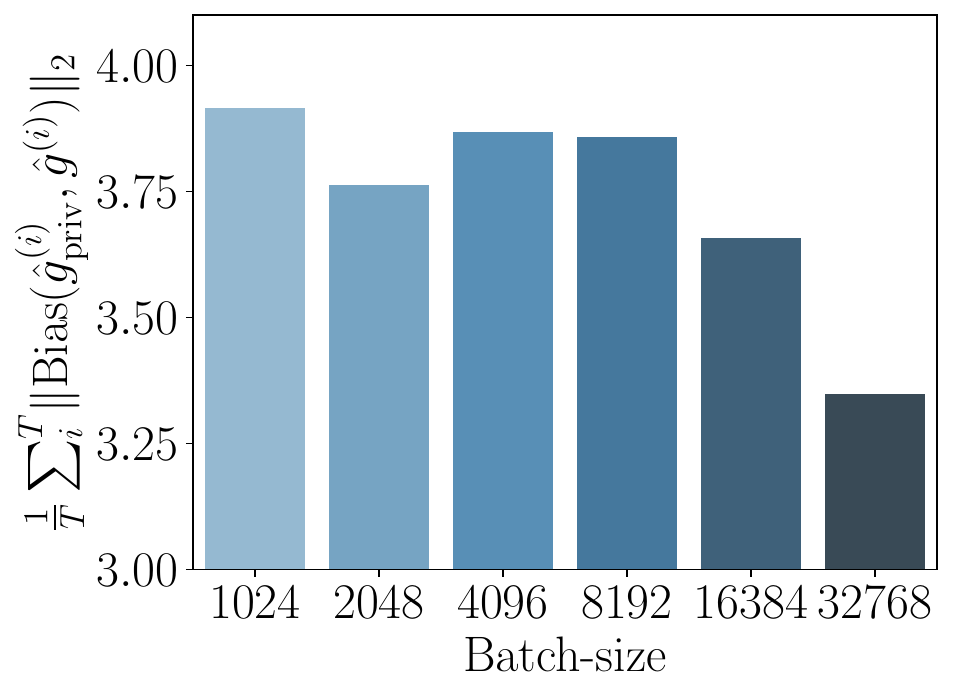}
    \caption{Bigger batches reduce private gradient bias magnitude. The figure showcases bias magnitude averaged over training iterations when training on CIFAR-10 with DP-SGD with different batch sizes.}
    \label{fig:bias_batch_size}
\end{figure}

\begin{figure}[H]
    \centering
    \includegraphics[width=0.65\linewidth]{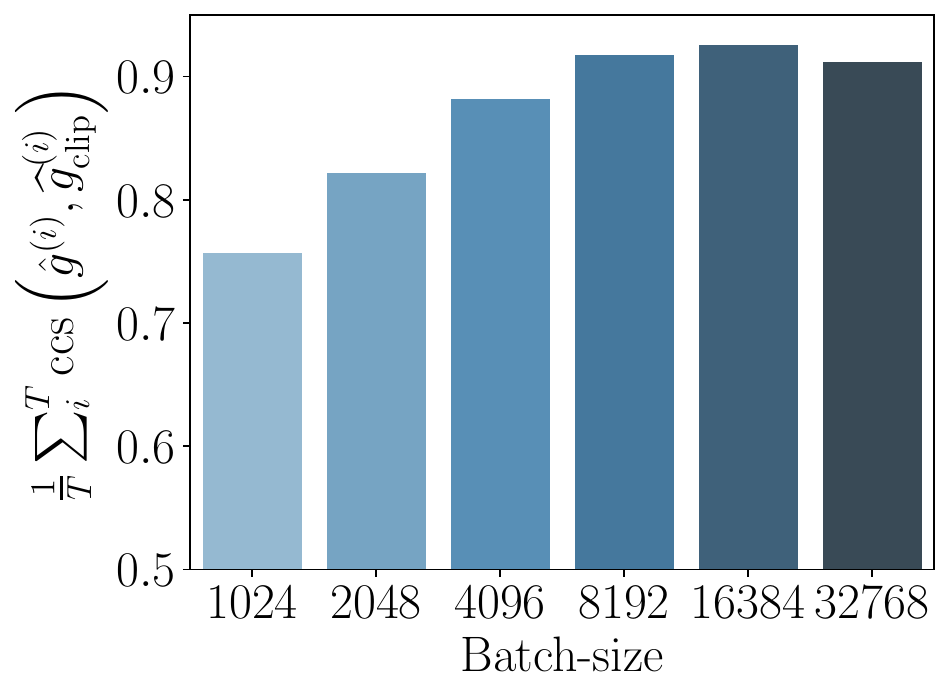}
    \caption{Bigger batches reduce the directional component of private gradient bias as measured by an increased average cosine similarity between clipped and unclipped mini-batch gradients during training on CIFAR-10.}
    \label{fig:align_batch_size}
\end{figure}

We find large batches to reduce both: the magnitude of the bias vector and the directional components of the bias vector. The latter is indicated by an increased cosine similarity between clipped and unclipped gradients.

\section{BAM and effects on loss landscape flatness}
\label{sec:bam_flatness}

Minimising \eqref{bias aware objective} has the additional benefit of smoothing the optimisation landscape around the minima (\textit{flat minima}) which, as originally argued by \citet{hochreiter1997flat}, is widely believed to offer better generalisation performance than sharp minima. To find such flat minima, \citet{zhao2022penalizing} propose the objective:

\begin{flalign}
\label{zhao 2022 objective}
\mathcal{L}_Z(\theta) &= \mathcal{L}(\theta) + \lambda \big\lVert \nabla_\theta \mathcal{L}(\theta)\big\rVert_2 \\
&= \mathcal{L}(\theta) + \lambda \Big\lVert \frac{1}{n}\sum_{i=1}^n g_i\Big\rVert_2
\end{flalign}

This objective finds flat minima because $\lVert \nabla_\theta \mathcal{L}(\theta)\rVert_2$ approximates the Lipschitz constant of the loss function $\mathcal{L}$. The Lipschitz constant is an upper bound on how much the magnitude of a change in the parameter space changes the magnitude of the loss. This is a way of capturing \say{flatness}: the smaller the Lipschitz constant of the loss function, the flatter the minima (since the same change in $\theta$ would lead to a smaller bound on the change in the loss magnitude). Hence, by decreasing the Lipschitz constant of the loss function, the solution will be driven towards a flatter minimum. 

However, objective (\ref{zhao 2022 objective}) cannot be used in our case because it does not allow for the per-sample analysis required for DP. Nonetheless, we can show the following:

\begin{lemma}
The bias aware objective  $\mathcal{L}_{\mathrm{BAO}}$ (\ref{bias aware objective}) upper-bounds the objective $\mathcal{L}_Z$ (\ref{zhao 2022 objective}) of \cite{zhao2022penalizing}.
\end{lemma}
\begin{proof}
This follows from the triangle inequality:

\begin{align}
\label{inequal flatness}
\mathcal{L}(\theta) + \lambda \Big\lVert\frac{1}{n}\sum_{i=1}^ng_i\Big\rVert_2 &= \mathcal{L}(\theta) + \lambda\frac{1}{n} \Big\lVert\sum_{i=1}^n g_i\Big\rVert_2 \\
&\leq \mathcal{L}(\theta) + \lambda \frac{1}{n}\sum_{i=1}^n\lVert g_i\rVert_2. 
\end{align}
\end{proof}

\noindent Thus, minimising $\mathcal{L}_{\mathrm{BAO}}$ will additionally drive the solution to a flat minimum by minimising an upper bound on (\ref{zhao 2022 objective}).

\section{A meaningful decomposition of the private gradient}
\label{sec:decomp}

\begin{figure}[H]
    \centering
    \hspace{5mm}
    \begin{minipage}{0.4\linewidth}
        \input{figs/appendix_bias}
    \end{minipage}
    \hspace{5mm}
    \begin{minipage}{0.4\linewidth}
        \input{figs/bias_decomposed}
    \end{minipage}
    \caption{Illustration of the proposed bias (a) decomposition into orthogonal vector components (b): magnitude \textcolor{my_orange}{$\boldsymbol{a}$} and direction \textcolor{purple!70!black}{$\boldsymbol{c}$}.}
    \label{fig:my_label}
\end{figure}
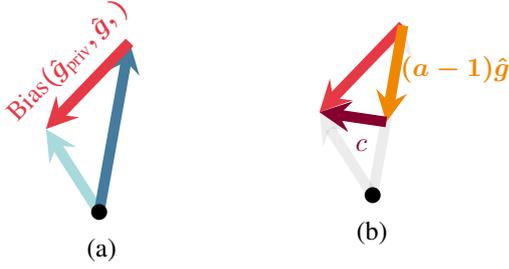

With the aim of better capturing the pathological nature of clipping during private stochastic optimisation, we propose a bias decomposition into orthogonal components that allow for the isolation of \textit{magnitude} and directional \textit{estimation} error.

In our case, we have an estimator $\hat{g}_{\mathrm{priv}}$ of $\hat{g}$. We now show a decomposition of $\hat{g}_{\mathrm{priv}}$ in terms of $\hat{g}$, which will allow us to differentiate harmful from harmless bias.

\begin{theorem}
\label{decomp theorem}
The private gradient estimate $\hat{g}_{\mathrm{priv}}$ can be decomposed as:

$$\hat{g}_{\mathrm{priv}} = a\cdot \hat{g} + c,$$

\noindent where $a \in \mathbb{R}, c \in \mathbb{R}^d$. We call $a$ the \textbf{magnitude error} and $c$ the \textbf{directional error} that arise through (private) estimation.
\end{theorem}

\begin{proof}
We first relate each per-sample gradient $g_i$ to the mini-batch gradient $\hat{g}$ through a vector decomposition into orthogonal components.
$$g_i = \mathrm{proj}_{\hat{g}}\,g_i + \tau_i$$
\noindent where $\mathrm{proj}_{\hat{g}}\,g_i = \frac{\langle g_i, \hat{g}\rangle}{\lVert \hat{g}\rVert^2} \cdot \hat{g}$ and  $\tau_i$ is a suitable vector orthogonal to the projection. Letting $\eta_i = \frac{\langle g_i, \hat{g}\rangle}{\lVert \hat{g}\rVert^2}$, the decomposition is:
\begin{equation}
    \mathrm{proj}_{\hat{g}}\,g_i = \eta_i \cdot \hat{g} + \tau_i.
\end{equation}

Utilising this decomposition inside the $\mathrm{clip}$ function gives us:
\begin{equation}
\mathrm{clip}(g_i) = \frac{g_i}{\max(1, \frac{C}{||\hat{g}_i||_2})} = \frac{\eta_i \hat{g} + \tau_i}{\max(1, \frac{C}{||g_i||_2})}.
\end{equation}
Now, let $M_i = \max(1, \frac{C}{||g_i||_2})$ so that we can simplify the above as
\begin{equation}
\mathrm{clip}(g_i) = \frac{\eta_i}{M_i} \hat{g} + \frac{\tau_i}{M_i}.
\end{equation}
Averaging the clipped gradients across the mini-batch yields:
\begin{align}
\hat{g}_{\mathrm{clip}} &= \frac{1}{l}\sum_{i=1}^l \mathrm{clip}(g_i) \\
&=\frac{1}{l}\sum_{i=1}^l\left(\frac{\eta_i}{M_i} \hat{g} + \frac{\tau_i}{M_i}\right) \\
 &= \left(\frac{1}{l}\sum_{i=1}^l\frac{\eta_i}{M_i}\right)\; \cdot \; \hat{g} + \left(\frac{1}{l}\sum_{i=1}^l\frac{\tau_i}{M_i}\right).
\end{align}
Finally, perturbing each $g_i$ with a vector $\beta_i \sim \mathcal{N}(0, C\sigma I_d)$ we obtain the desired decomposition in terms of magnitude error $a$ and directional error $c$:
\begin{align}
\hat{g}_{\mathrm{priv}} &= \frac{1}{l}\sum_{i=1}^l\left(\mathrm{clip}(g_i) + \beta_i\right) \\
&= \underbrace{\left(\frac{1}{l}\sum_{i=1}^l\frac{\eta_i}{M_i}\right)}_{a}\; \cdot \; \hat{g} \; + \; \underbrace{\left(\frac{1}{l}\sum_{i=1}^l\frac{\tau_i}{M_i}+\beta_i \right)}_{c}
\end{align}
\vspace{-4mm}
\end{proof}

This allows us to now effectively isolate directional from magnitude components of the private gradient bias:

\begin{align}
    \mathrm{Bias}(\hat{g}_{\mathrm{priv}}, \hat{g}) &= \E \left[ \hat{g}_{\mathrm{priv}} \right] - \hat{g} \\
    &= \E \left[ a \cdot \hat{g} + c\right] - \hat{g}.
\end{align}



\end{document}

%% file: figs/clip_a.tex
\usetikzlibrary{angles,quotes}

\tikzset{
dot/.style = {circle, fill, minimum size=#1,
              inner sep=0pt, outer sep=0pt},
dot/.default = 6pt 
}

\begin{tikzpicture}

    \draw[line width=4pt, my_darkblue,-stealth](0,0)--(2.45, 4) node[anchor=west]{$\boldsymbol{g_1}$};
    \draw[line width=4pt, my_red,-stealth](0,0)--(1.04, 1.7) node[anchor= west]{$\text{clip}(\textcolor{black}{\boldsymbol{g_1}})$};
    \draw[line width=4pt, my_red, -stealth](0,0)--(-1.95,0.485) node[anchor= east]{$\text{clip}(\textcolor{black}{\boldsymbol{g_2}})$};
    
    \draw[line width=4pt, my_blue,-stealth](0,0)--(0.4, 2.25) node[anchor=west, yshift=1mm, xshift=-1mm]{$\boldsymbol{\hat{g}}$};
    \draw[line width=4pt, my_lightblue,-stealth](0,0)--(-0.7, 1.1) node[anchor=east, yshift=-2mm, xshift=1mm]{$\boldsymbol{\hat{g}_{\text{clip}}}$};
    
    
 
    \begin{scope}
        \clip (-2.5,0) rectangle (3.0, 3.0);
        \draw [line width=4pt, my_red, -stealth, draw opacity=0.1] (0,0) circle(2.0);
    \end{scope}

    \node[dot=6pt] at (0,0) {};

    \node[text width=1cm] at (0.3,-0.5) 
    {(a)};
\end{tikzpicture}

%% file: figs/clip_b.tex
\usetikzlibrary{angles,quotes}
\usetikzlibrary{arrows.meta,decorations.pathreplacing}

\tikzset{
dot/.style = {circle, fill, minimum size=#1,
              inner sep=0pt, outer sep=0pt},
dot/.default = 6pt 
}

\begin{tikzpicture}

    \draw[line width=4pt, my_lightblue, -stealth, opacity=1.0](0,0)--(-0.7, 1.1) node[anchor=north east]{};
    \draw[line width=4pt, my_blue,-stealth, opacity=1.0](0,0)--(0.4, 2.25) node[anchor= north west]{};

    \draw[line width=4pt, my_red,-stealth](0.4, 2.25)--(-0.7, 1.1) node[anchor=south]{};
    
    \node[midway, xshift=-4mm,yshift=20mm]{\rotatebox{46}{\textcolor{my_red}{$\boldsymbol{\text{Bias}(\hat{g}_\text{priv}, \hat{g},)}$}}};
    
    \node[dot=6pt, opacity=1] at (0,0) {};

    \node[text width=1cm] at (0.35,-0.5) 
    {(b)};
\end{tikzpicture}

%% file: figs/appendix_bias.tex
\usetikzlibrary{angles,quotes}
\usetikzlibrary{arrows.meta,decorations.pathreplacing}

\tikzset{
dot/.style = {circle, fill, minimum size=#1,
              inner sep=0pt, outer sep=0pt},
dot/.default = 6pt 
}

\begin{tikzpicture}

    \draw[line width=4pt, my_lightblue, -stealth, opacity=1.0](0,0)--(-0.7, 1.1) node[anchor=north east]{};
    \draw[line width=4pt, my_blue,-stealth, opacity=1.0](0,0)--(0.4, 2.25) node[anchor= north west]{};

    \draw[line width=4pt, my_red,-stealth](0.4, 2.25)--(-0.7, 1.1) node[anchor=south]{};
    
    \node[midway, xshift=-4mm,yshift=20mm]{\rotatebox{46}{\textcolor{my_red}{$\boldsymbol{\text{Bias}(\hat{g}_\text{priv}, \hat{g},)}$}}};
    
    \node[dot=6pt, opacity=1] at (0,0) {};

    \node[text width=1cm] at (0.35,-0.5) 
    {(a)};
\end{tikzpicture}

%% file: figs/bias_decomposed.tex
\usetikzlibrary{angles,quotes}

\tikzset{
dot/.style = {circle, fill, minimum size=#1,
              inner sep=0pt, outer sep=0pt},
dot/.default = 6pt 
}

\begin{tikzpicture}
 
   \draw[line width=4pt, gray, -stealth, opacity=0.15](0,0)--(0.4, 2.25) node[anchor= south]{};
   \draw[line width=4pt, gray, -stealth, opacity=0.15](0,0)--(-0.7, 1.1) node[anchor=south]{};

    \draw[line width=4pt, my_red, -stealth](0.4, 2.25)--(-0.7, 1.1) node[anchor=south]{};
    
    \draw[line width=4pt, my_orange,-stealth](0.4, 2.25)--(0.18, 0.97) node[anchor=west, yshift=7mm]{$\boldsymbol{(a-1)\hat{g}}$};

    \draw[line width=4pt, purple!70!black, -stealth] (0.18, 0.97)--(-0.7, 1.1) node[anchor=south, xshift=5.5mm, yshift=-7mm]{$c$};

    \node[dot=6pt, opacity=1.0] at (0,0) {};
    
    \node[text width=1cm] at (0.3,-0.5){(b)};
 \end{tikzpicture}